\documentclass[10pt,doublecolumn]{IEEEtran}

\usepackage{amsmath}
\usepackage{amssymb}
\usepackage{amsthm}
\usepackage{epsfig}

\usepackage{graphics}

\renewcommand{\vec}[1]{\ensuremath{{\rm\bf{#1}}}}
\newcommand{\mat}[1]{\ensuremath{{\rm\bf{#1}}}}

\newtheorem{theorem}{Theorem}

\begin{document}
%

\title{Bayesian Kernel and Mutual $k$-Nearest Neighbor Regression}

%
%
%

\author{Hyun-Chul Kim
\thanks{H.-C. Kim is with R$^2$ Research,
       Seoul, South Korea.    \newline
        E-mail: hckim.sr{\fontfamily{ptm}\selectfont @}gmail.com}  } 

\maketitle

\begin{abstract}
We propose Bayesian extensions of two nonparametric regression methods which are kernel and mutual $k$-nearest neighbor regression methods.  Derived based on Gaussian process models for regression, the extensions provide distributions for target value estimates and the framework to select the hyperparameters. It is shown that both the proposed methods asymptotically converge to  kernel and mutual $k$-nearest neighbor regression method, respectively. The simulation results show that the proposed methods can select proper hyperparameters and are better than or comparable to the former methods for an artificial data set and a real world data set.
\end{abstract}


%

\section{Introduction}
  In regression analysis, it is analyzed how the response variable $y \in \mathbf{R}$ depends on the value of the observation vector $\vec x \in \mathbf{R}^d$ \cite{opac-b1123996}. The classical approach for regression is a parametric regression approach, where a certain type of structure of the regression function is assumed to be known and can be described by finitely many parameters. A big drawback of the parametric regression is that it cannot approximate the function better than the best one in the assumed parametric structure. This drawback can be avoided by the nonparametric regression approach, where does not assume that the regression function can be determined by finitely many parameters.

One of well-known nonparametric regression methods is kernel regression \cite{opac-b1123996}.
Kernel regression estimator has kernel function with an essential parameter called bandwidth.  The performance of kernel regression is known to be largely dependent on the selection of the parameter called bandwidth \cite{herrmann2000variance}.  Methods based on cross-validation for bandwidth selection in kernel regression has been proposed \cite{hardle1985optimal,hardle1997multivariate}. It was shown to be asymptotically optimal.  A Bayesian approach of averaging over bandwidth rather than selecting bandwidth has been also proposed \cite{zhang2009bayesian}. They set a leave-one-out likelihood and a prior for bandwidth, and derived the posterior estimate of bandwidth. Rather than selecting a single bandwidth, they averaged the estimator over bandwidth. Metropolis-Hastings algorithm was used for sampling bandwidth to approximate the posterior. 

Kernel regression has been widely applied to various kinds of areas including the empirical finance.  It was applied to estimate state-densities implicit in financial asset prices \cite{ait1998nonparametric,ait2000nonparametric}. It was
also applied to estimate call prices and exercise decisions \cite{broadie2000american}, option prices  \cite{aït2001goodness}, and the probability of a crash \cite{fernandes2006financial}.
Besides the empirical finance area, kernel regression and its data-adapted extensions have been  applied to image processing, restoration, and enhancement \cite{Takeda07kernelregression}. Bilateral fitlering widely used in image processing is a data-adapted extension of kernel regression which uses image pixel values as inputs of  kernel function \cite{Takeda07kernelregression,tomasi1998bilateral,elad2002origin}.  Nonlocal-means filtering for super-resolution reconstruction is a data-adapted extension of kernel regression which uses pixel locations and extracted image patches as inputs of kernel function \cite{protter2009generalizing}. Another data-dependent extension of kernel regression is steering kernel regression which uses pixel locations and local gradients in the image as inputs of kernel function \cite{Takeda07kernelregression}.

 Another of well-known nonparametric regression methods is $k$-nearest neighbor ($k$-NN) regression \cite{opac-b1123996}.  Mutual $k$-NN (M$k$NN) regression is a variate of $k$-NN regression based on mutual neighborship rather than one-sided neighborship \cite{JMLR:v14:guyader13a}.  Even though it was not applied to regression task, M$k$NN concept existed even in 1970s. \cite{gowda1978agglomerative,gowda1979condensed} has used M$k$NN methods for clustering. More recently, M$k$NN methods have been applied to classification \cite{liu2010new}, outlier detection \cite{hautamaki2004outlier}, object retrieval \cite{jegou2010accurate}, and clustering of interval-valued symbolic patterns \cite{guru2006clustering}. \cite{ozaki2011using} used M$k$NN concept to semi-supervised classification of natural language data and showed that the case of
using M$k$NN concept consistently outperform the case of using $k$-NN concept. For the regression case, \cite{JMLR:v14:guyader13a} has argued that M$k$NN regression might be less influenced by hubs which are data points appearing in the nearest neighbor list of many data points. 

In this paper we propose a novel method for Bayesian kernel regression and a Bayesian method for mutual $k$-NN regression{\footnote{To our knowledge, no Bayesian approach for either $k$-NN regression or mutual $k$-NN regression has been proposed.}}. It provides not only the distribution of target value, but also
Bayesian model selection framework for bandwidth in kernel regression and $k$ in mutual $k$-NN regression. We take Laplacian-based covariance matrix and use Gaussian process model. We show that the mean of target values in the proposed methods is asymptotically the estimates of kernel regression and mutual $k$-NN regression, respectively

The paper is organized as follows.  In section \ref{sec:kernel_MkNN_reg} we describe kernel regression and mutual $k$-NN regression. In section \ref{sec:Bayes_kernel_MkNN}  we briefly explain Gaussian process regression with Laplacian-based covariance matrix, and based on that model propose Bayesian kernel regression and Bayesian mutual $k$-NN regression. In the section we also propose the Bayesian model selection methods for both cases. In section \ref{sec:sim_res} we show simulation results for an artificial data set and a real-word data set. In section \ref{sec:related_works} we discuss the works related to the proposed models. Finally a conclusion is drawn.

\section{Kernel and Mutual $k$-Nearest Neighbor Regression}
\label{sec:kernel_MkNN_reg}
\subsection{Kerenel Regression}

Let us assume that we have the set of pairs of the observation vector and response variable $\mathcal{D}_n=\{ (\vec x_1, y_1), \ldots, (\vec x_n, y_n) \}$, where $\vec x_i \in \mathbf{R}^d$ and $y_i \in \mathbf{R}$. Given a new observation vector $\vec x$, the kernel estimate is defined as follows. 
\begin{align}
m^{\mathrm{ker}}_n(\vec x) = 
\left\{ \begin{array}{ll}
         \frac{\sum_{i=1}^n y_i k(\frac{\vec x - \vec x_i}{h_n})}
				{\sum_{i=1}^n k(\frac{\vec x - \vec x_i}{h_n})}
 & \mbox{if $\sum_{i=1}^n k(\frac{\vec x - \vec x_i}{h_n}) \neq 0$};\\
        0 & \mbox{otherwise,} \end{array} \right.
\label{eqn:ker_reg}
\end{align}
where $h_n$ is a bandwidth only depending on the sample size $n$ and $k(\cdot)$ is a kernel function which maps $\mathcal{R}^d$ to $ \mathcal{R}^{+}\cup \{ 0 \}$ \cite{opac-b1123996}.  In this paper we use the following kernel function.
\begin{align}
	k(\vec z) = \exp ( - || \vec z  ||^2 ).
\end{align}

In the kernel function $k(\frac{\vec x_i - \vec x_j}{h_n})$, a single identical bandwidth for all the dimensions is used.  In addition to the kernel function with a single bandwidth we use another kernel function, which uses an individual bandwidth for each dimension, $k((\vec x_i - \vec x_j) {\mat H}^{-1})$ where $\mat H$ is a $d\times d$ diagonal matrix with the diagonal elements $h_{n,1}, h_{n,2}, \ldots, h_{n,d}$. The former one is called the single bandwidth scheme, and the latter one is called the multiple bandwidth scheme. For the convenience, we also denote $k(\frac{\vec x_i - \vec x_j}{h_n})$, or $k((\vec x_i - \vec x_j) {\mat H}^{-1})$ by $k(\vec x_i, \vec x_j)$, and we mention whether it is under the single bandwidth scheme or the multiple bandwidth scheme.

\subsection{Mutual $k$-Nearest Neighbor Regression}

First, we describe $k$-nearest neighbor regression method \cite{opac-b1123996}. We denote a reordering of the elements of $\mathcal{D}_n$ by $(\vec x_{(1,n)}(\vec x),y_{(1,n)}(\vec x)),\ldots,(\vec x_{(n,n)}(\vec x),y_{(n,n)}(\vec x) )$ according to increasing values of $||\vec x_i - \vec x||$, Euclidean distance between $\vec x_i$ and $\vec x$. Then, given $\vec x \in \mathbf{R}^d$, the $k$-Nearest Neighbor ($k$-NN) estimate is defined by
\begin{align}
   m^{k\mbox{NNR}}_n(\vec x) = \frac{1}{k}\sum_{i=1}^k y_{(i,n)}(\vec x).
\end{align}

 Mutual $k$-nearest neighbor regression method  \cite{JMLR:v14:guyader13a} is a method to consider $k$ mutually nearest neighbor data points, rather than just $k$ nearest neighbor data points.  Let $\mathcal{N}_k(\vec x)$ be the set of the $k$ nearest neighbors of $\vec x$ in $\mathcal{D}_n$, $\mathcal{N}'_k(\vec x_i)$ the set of $k$ nearest neighbors of $\vec x_i$ in $(\mathcal{D}_n \backslash \{\vec x_i \}) \cup \{ \vec x \}.$ The set of Mutual $k$-Nearest Neighbors (M$k$NNs) of $\vec x$ is defined as
\begin{align}
	\mathcal{M}_k(\vec x) = \{ \vec x_i \in \mathcal{N}_k(\vec x) : \vec x \in \mathcal{N}'_k (\vec x_i) \}.
\end{align}
Then, the mutual $k$-nearest neighbor regression estimate is defined as
\begin{align}
    m^{\mbox{M$k$NNR}}_n(\vec x) = \left\{ \begin{array}{ll}
         \frac{1}{M_k(\vec x)}\sum_{i:\vec x_i \in \mathcal{M}_k(\vec x)}^k y_i & \mbox{if $M_k(\vec x) \neq 0$};\\
        0 & \mbox{if $M_k(\vec x) = 0$}.\end{array} \right.
\label{eqn:mut_nn_reg_est}
\end{align}
where $M_k(\vec x) = | \mathcal{M}_k(\vec x)| $. 

\section{Bayesian Kernel and Mutual $k$-NN Regression via Gaussian Processes}
\label{sec:Bayes_kernel_MkNN}
\subsection{Gaussian Process Regression}

Assume that we have a data set $D$ of data points $\vec x_i$ with
continuous target values $y_i$: $D = \{(\vec x_i, y_i)|i = 1, 2,
\ldots, n\}$, $X = \{\vec x_i|i = 1, 2, \ldots, n \}$, $\vec y=[y_1,
y_2, \ldots, y_n]^T$. We assume that the observations of target values are nosiy,
and set $y_i=f(\vec x_i)+\epsilon_i$, where $f(\cdot)$ is a target function to be estimated
and $\epsilon_i \sim \mathcal{N}(0,v_1)$. A function $f(\cdot)$ to be estimated given
$D$ is assumed to have Gaussian process prior, which means that any
collection of functional values are assumed to be multivariate
Gaussian \cite{Williams95gaussian,Rasmussen:2005:GPM:1162254}.

The prior for the function values $\vec f$ $(=[f(\vec x_1), f(\vec x_2),  $ $\ldots, f(\vec x_n) ]^T)$ is assumed to be Gaussian:
\begin{align}
p(\vec f|X, \Theta_f) = \mathcal{N}(\vec 0, \mat C_f).
\end{align}
Then the density function
for the target values can be described as follows.
\begin{align}
	p(\vec y| X, \Theta) &= \mathcal{N}(\vec 0, \mat C_f + v_1 \mat I) \\
					&= \mathcal{N}(\vec 0, \mat C),
\end{align} 
where $\mat C$ is a matrix whose elements $C_{ij}$ is a covariance
function value $c(\vec x_i,\vec x_j)$ of $\vec x_i$, $\vec x_j$ and
$\Theta$ is the set of hyperparameters in the covariance function. 

One of the widely used covariance functions is as follows:
\begin{align}
c(\vec x_i, \vec x_j) =& v_0 \exp \{ -\frac{1}{2}\sum_{m=1}^d l_m
(x_i^m-x_j^m)^2 \} + v_1\delta(i,j), \label{eqn:cov_ftn}
\end{align}
where $x_i^m$ is the $m$th element of $\vec x_i$. The hyperparameter
$v_0$ specifies the overall vertical scale of variation of the
target values, $v_1$ the noise variance of the target values,
and $l_m$ the (inverse) lengthscale for feature dimension $m$. The
covariance function described in Eq (\ref{eqn:cov_ftn}) enables GPR to
estimate a nonlinear function. $v_1$ makes GPR robust to noise, and
the optimized $l_1, \ldots, l_d$ can be used to determine relevant
features.

It can be shown that GPR provides the following
distribution of target value $f_{\mathrm{new}}(=f(\vec x_{\mathrm{new}}))$ given a
test data $\vec x_{\mathrm{new}}$:
\begin{align}
 p(f_{\mathrm{new}}|\vec x_{\mathrm{new}}, D,\Theta)=\mathcal{N}(\vec k^T \mat
 C^{-1} \vec f, \kappa - \vec k^T \mat C^{-1} \vec k),
 \label{eqn:result}
\end{align}
where $\vec k = [c(\vec x_{\mathrm{new}},\vec x_1) \ldots c(\vec
x_{\mathrm{new}},\vec x_n)]^T$, $\kappa=c(\vec x_{\mathrm{new}},\vec x_{\mathrm{new}})$. The
variance of the target value $f_{\mathrm{new}}$ is related to the degree of
its uncertainty.

Since the set of hyperparameters $\Theta(=\{v_0,v_1,l_1,\ldots,l_d\})$
controls the model complexity, it is very important to select the
most proper $\Theta$. The proper $\Theta$ can be obtained by
maximizing the marginal likelihood $p(\vec y|X,\Theta)$
\cite{Williams95gaussian,Gibbs97efficientimplementation,Rasmussen:2005:GPM:1162254}.
\begin{align}
    \log p(\vec y | \mat X, \Theta) = -\frac{1}{2}\vec y^\top \mat C^{-1} \vec y - \frac{1}{2} \log | \mat C | - \frac{N  T}{2}  \log 2 \pi.
\label{eqn:loglikeli}
\end{align}
Rather than choosing a single set of hyperparameters by optimization, we can average
over the hyperparameters with MCMC methods \cite{Williams95gaussian,Neal1997RC_GP}.

\subsection{Laplacian-based Covariance Matrix}

The combinatorial Laplacian $\mat L$ is defined as follows.
\begin{align}
\mat L = \mat D - \mat W,\label{eq:Delta}
\end{align}where $\mat W$ is an $N\times N$ edge-weight matrix with the edge
weight between two points $\vec x_i$,$\vec x_j$ given as $w_{ij}(=w(\vec x_i,\vec x_j))$ and $\mat D =
\mathrm{diag}(d_1,...,d_N)$ is a diagonal matrix with diagonal
entries $d_i=\sum_{j} w_{ij}$.

Similarly to \cite{zhu2003semi}, to avoid the singularity we use
 Laplacian-based covariance matrix as
\begin{align}
   \mat C =&  (\mat L + \sigma^2 \mat I)^{-1} = \tilde{\mat C}^{-1}.
\end{align}
Then, we have Gaussian process prior as follows.
\begin{align}
	p(\vec y| X, \Theta) &= \mathcal{N} (\vec 0, \mat C),
\end{align}
The predictive distiribution for $y_{\mathrm{new}}$ is as follows.
\begin{align}
	p(y_{\mathrm{new}}| \vec y, X, \vec x_{\mathrm{new}}, \Theta) &= \frac{p(\vec y_{\mathrm{new}}| X, \vec x_{\mathrm{new}}, \Theta)}{p(\vec y| X, \vec x_{\mathrm{new}}, \Theta)}
	\nonumber \\
	&= \mathcal{N}_{\vec y_{\mathrm{new}}} (\vec 0, \mat C_{\mathrm{new}}) / \mathcal{N}_{\vec y} (\vec 0, \mat C) \nonumber \\
				&= \mathcal{N}_{\vec y_{\mathrm{new}}} (\vec 0, \tilde{\mat C}_{\mathrm{new}}^{-1}) /  \mathcal{N}_{\vec y} (\vec 0, \tilde{ \mat C}^{-1}) \nonumber \\
				&\propto \frac{\exp(- \frac{1}{2} [\vec y^T y_{\mathrm{new}}] 
				\tilde{\mat C}_{\mathrm{new}} [\vec y^T y_{\mathrm{new}}]^T )}
				{\exp(- \frac{1}{2} \vec y^T
				\tilde{\mat C} \vec y )} \nonumber \\
				&\propto \exp(- \frac{1}{2}\tilde{\kappa} y_{\mathrm{new}}^2 - \tilde{\vec k}^T \vec y y_{\mathrm{new}}) \nonumber \\
				&\propto  \mathcal{N}(-\frac{1}{\tilde{\kappa}} \tilde{\vec k}^T \vec
 y,\frac{1}{\tilde{\kappa}}),
\end{align}
where
\begin{align}
\tilde{\mat C}_{\mathrm{new}}&=\left[\begin{array}{ll}
\mathbf{\tilde{C}} & \tilde{\vec k} \\
   \tilde{\vec k}^T & \tilde{\kappa}
   \end{array} \right]
=\mathbf{C}_{\mathrm{new}}^{-1},  \\
\label{aaa}
   \tilde{\kappa} &= \sum_{i=1}^N w(\vec x_{\mathrm{new}},\vec x_i) + \sigma^2, \\
   \tilde{ \vec k}^T &= - [ w(\vec x_{\mathrm{new}}, \vec x_1), w(\vec x_{\mathrm{new}}, \vec x_2), \ldots, w(\vec x_{\mathrm{new}}, \vec x_N) ].
\end{align}
The mean and variance of $y_{\mathrm{new}}$ is represented as
\begin{align}
   \mu_{ y_{\mathrm{new}}}&= -\frac{1}{\tilde{\kappa}} \tilde{\vec k}^T \vec y_L 
                            =\frac{\sum_{i=1}^N w (\vec x_{\mathrm{new}}, \vec x_i) y_i }{\sum_{i=1}^N w (\vec x_\mathrm{new}, \vec x_i)+\sigma^2}, \label{eq:transd_lap_mean} \\
   \sigma^2_{ y_{\mathrm{new}}}&= \frac{1}{\tilde{\kappa}}
                            =\frac{1}{\sum_{i=1}^N w(\vec x_{\mathrm{new}},\vec x_i) + \sigma^2}. \label{eq:transd_lap_std}
\end{align}
Eq (\ref{eq:transd_lap_mean}) links Gaussian process regression with Laplacian-based covariance matrix to kernel and mutual $k$-NN regression, which will be described below. 

\subsection{Bayesian Kernel Regression}

If we set $w_{ij}=w_{\mathrm{ker}} (\vec x_i, \vec x_j)=\sigma_0 k(\vec x_i, \vec x_j)$, where $k(\cdot)$ is a kernel function in Eq (\ref{eqn:ker_reg}) and $\sigma_0>0$,
we get the following theorem for the validity of the covariance matrix for Gaussian processes.  
\begin{theorem}
   Covairance matrix $\tilde{\mat C}$ with $w_{ij}=w_{\mathrm{ker}} (\vec x_i, \vec x_j)$ 
is valid for Gaussian processes if $\sigma^2>0$.
\label{thm:ker_cov_valid}
\end{theorem}
\begin{proof}
   (1) Since Laplacian matrix $\mat L(=\mat D - \mat W)$ is positive semidefinite \cite{merris1994laplacian}, for $\sigma^2>0$ $\tilde{\mat C}(=\mat L + \sigma^2 \mat I)$ is positive definite. So $\tilde{\mat C}$ is positive definite.  \\
   (2) Since $\tilde{\mat C}^T = (\mat D - \mat W + \sigma^2 \mat I)^T
   		= \mat D^T - \mat W^T + \sigma^2 \mat I^T = \mat D - \mat W + \sigma^2 \mat I
	        = \tilde{\mat C}$,
	$\tilde{\mat C}$ is symmetric. \\
   From (1) \& (2), by Theorem 7.5 in \cite{stefanica2014linear} $\tilde{\mat C}$  is a valid covariance matrix. QED.
\end{proof}

By applying Eq (\ref{eq:transd_lap_mean}), Bayesian kernel regression estimate for $\vec x_{\mathrm{new}}$ is defined as follows:
\begin{align}
m^{\mathrm{Bker}}_n(\vec x_{\mathrm{new}}) = 
\mu_{y_{\mathrm{new}},\mathrm{ker}},
\end{align}
where
\begin{align}
	\mu_{y_{\mathrm{new}},\mathrm{ker}} &=  \frac{\sum_{i=1}^N w_\mathrm{ker} (\vec x_{\mathrm{new}}, \vec x_i) y_i }{\sum_{i=1}^N w_{\mathrm{ker}} (\vec x_\mathrm{new}, \vec x_i)+\sigma^2} \nonumber \\
		&= \frac{\sum_{i=1}^N k (\vec x_{\mathrm{new}}, \vec x_i) y_i }{\sum_{i=1}^N  k (\vec x_\mathrm{new}, \vec x_i)+\sigma^2/\sigma_0}.
\label{eqn:Bayes_ker_reg}
\end{align}

The following theorem shows that Bayesian kernel regression introduced above asymptotically converges to the traditional kernel regression.
\begin{theorem}
  $\mu_{y_{\mathrm{new}},{\mathrm{ker}}} (=-\frac{1}{\tilde{\kappa}} \tilde{\vec k}^T \vec y)$ converges to kernel regression as $\sigma^2/\sigma_0$ approaches $0$.
\end{theorem}
\begin{proof}
  In case $\sum_{i=1}^N  k (\vec x_\mathrm{new}, \vec x_i) = 0$, $ k (\vec x_\mathrm{new}, \vec x_i) = 0$ for all $i$.  So it is trivial by Eq (\ref{eqn:ker_reg}) and (\ref{eqn:Bayes_ker_reg}).

  Otherwise, take a small positive $\epsilon< m^{\mathrm{ker}}_n(\vec x)$.
  Set $\delta = \{ \sum_{i=1}^N k (\vec x, \vec x_i) \} / \{ \frac{m^{\mathrm{ker}}_n(\vec x)}{\epsilon} -1 \}$.
  Then, 
    if $||\sigma^2/\sigma_0||<\delta $,
                $ || \mu_{\vec f_U,\mathrm{M}{\it k}\mathrm{NN}} - m^{\mathrm{ker}}_n(\vec x)  || < \epsilon $.      By the $(\epsilon,\delta)$ definition of the limit of a function, we get the statement in the theorem.
   QED.
\end{proof}

\subsection{Bayesian Mutual $k$-NN Regression}

Gaussian processes with Laplacian-based covariance matrix can be associated with mutual $k$-NN regression, by replacing $w_{ij}(=w(\vec x_i,\vec x_j))$ with the function 
\begin{align}
  w_{\mathrm{M}\it{k}\mathrm{NN}}(\vec x_i,\vec x_j) =  \sigma_0 \delta_{\vec x_j \sim_k \vec x_i} \cdot \delta_{\vec x_i \sim_k \vec x_j},
\end{align}
where the relation $\sim_k$ is defined as
\begin{align}
	\vec x_i \sim_k \vec x_j = \left\{ \begin{array}{ll}
 T  & \mbox{if $j \neq i$ and $\vec x_j$ is a $k$-nearest neighbor} \\
    & \mbox{of $\vec x_i$}; \\
   F   & \mbox{otherwise,} \end{array} \right.  
\end{align}

Similarly to Bayesian kernel regression, we get the following theorem related to the validity of Laplacian-based covariance matrix for Gaussian processes.
\begin{theorem}
   Covairance matrix $\tilde{\mat C}$ with $w_{ij}(=w_{\mathrm{M}\it{k}\mathrm{NN}}(\vec x_i,\vec x_j))$ is valid for Gaussian processes if $\sigma^2>0$.
\end{theorem}
\begin{proof}
 Similarly to the proof of Theorem \ref{thm:ker_cov_valid}, it can be shown that $\tilde{\mat C}$
 is positive definite and symmetric.
\end{proof}

By applying Eq (\ref{eq:transd_lap_mean}) like in Bayesian kernel regression,  Bayesian mutual $k$-NN regression estimate
for a given data $\vec x_{\mathrm{new}}$ is defined as follows.
\begin{align}
   m^{\mbox{BM}k\mbox{NNR}}_n(\vec x_{\mathrm{new}}) =    	\mu_{f_{\mathrm{new}},\mathrm{M}\it{k}\mathrm{NN}},
\end{align}
where
\begin{align}
   \mu_{f_{\mathrm{new}},\mathrm{M}\it{k}\mathrm{NN}}
=&\frac{\sum_{i=1}^N w_{\mathrm{M}\it{k}\mathrm{NN}} (\vec x_{\mathrm{new}}, \vec x_i) y_i }
{\sum_{i=1}^N w_{\mathrm{M}\it{k}\mathrm{NN}} (\vec x_{\mathrm{new}}, \vec x_i)+\sigma^2} \nonumber \\
 =&\frac{\sum_{i=1}^N  \delta_{\vec x_j \sim_k \vec x_i} \cdot \delta_{\vec x_i \sim_k \vec x_j} y_i }
{\sum_{i=1}^N \delta_{\vec x_j \sim_k \vec x_i} \cdot \delta_{\vec x_i \sim_k \vec x_j} +\sigma^2/\sigma_0 }.
\label{eqn:Bayes_mut_knn_reg}
\end{align}

The following theorem shows that Bayesian mutual $k$-nearest neighbor regression introduced above asymptotically converges to the traditional mutual $k$-NN regression.
\begin{theorem}
  $\mu_{f_{\mathrm{new}},{\mathrm{M}{\it k}\mathrm{NN}}} (=-\frac{1}{\tilde{\kappa}} \tilde{\vec k}^T \vec y)$ converges to mutual $k$-$\mathrm{NN}$ regression as $\sigma^2/\sigma_0$ approaches $0$.
\end{theorem}
\begin{proof}
  In case $\sum_{i=1}^N \delta_{\vec x_j \sim_k \vec x_i} \cdot \delta_{\vec x_i \sim_k \vec x_j} = 0$,  $\delta_{\vec x_j \sim_k \vec x_i} \cdot \delta_{\vec x_i \sim_k \vec x_j} = 0$ for all $i$. So it is trivial by Eq (\ref{eqn:mut_nn_reg_est}) and (\ref{eqn:Bayes_mut_knn_reg}). 

  Otherwise, take a small positive $\epsilon< m^{\mbox{M}k\mbox{NNR}}_n(\vec x)$. \\
  Set $\delta = \{ \sum_{i=1}^N  \delta_{\vec x_j \sim_k \vec x_i} \cdot \delta_{\vec x_i \sim_k \vec x_j}  \} / \{ \frac{m^{\mbox{M}k\mbox{NNR}}_n(\vec x)}{\epsilon} -1 \}$.
  Then, 
    if $||\sigma^2/\sigma_0||<\delta $,
                $ || \mu_{\vec f_U,\mathrm{M}{\it k}\mathrm{NN}} - m^{\mbox{M}k\mbox{NNR}}_n(\vec x)  || < \epsilon $.      By the $(\epsilon,\delta)$ definition of the limit of a function, we get the statement in the theorem.
   QED.
\end{proof}

\subsection{Hyperparameter Selection}

There is a set of hyperparameters that should be selected in both the proposed methods. For Bayesian kernel regression,  the set of hyperparameters is $\Theta=\{ h_n, \sigma_0, \sigma  \}$ for the single bandwidth scheme, or $\Theta=\{ h_{n,1}, h_{n,2}, \ldots, h_{n,d}, \sigma_0, \sigma  \}$ for the multiple bandwidth scheme.  For Bayesian $k$-NN regression, the set of hyperparameters is $\Theta=\{ k, \sigma_0, \sigma  \}$, where $k$ is a interger greater than $0$.  These sets of hyperparameters can be selected through the Bayesian evidence framework by maximizing the log of the marginal likelihood \cite{Gibbs97efficientimplementation} as follows.
\begin{align}
    \Theta^*
    	=& \mbox{argmax}_{\Theta} \mathcal{L}(\Theta),
\end{align}
where 
\begin{align}
\mathcal{L}(\Theta) =& \log p(\vec y|\Theta) \\
            =& \log \{ |2\pi \mat C|^{-\frac{1}{2}} \exp (-\frac{1}{2}\vec y^T \mat C^{-1} \vec
y) \} \\
            =& \frac{1}{2} \log |\tilde{\mat C}| - \frac{1}{2} \vec y^T \tilde{\mat C} \vec
y - \frac{N}{2} \log 2\pi, \label{eqn:bkr_log_ev}
\end{align}
where $\tilde{\mat C}=\mat L + \sigma^2 \mat I$.

The discrete hyperparameters (e.g., $k$) can be
selected based on the value of $\mathcal{L}$ as 
\begin{align}
      K^*	&= \mbox{argmax}_k  \mathcal{L}(\{k,\sigma,\sigma_0\}). 
\end{align}
For the continuous hyperparameters (e.g., $\sigma$, $\sigma_0$, $h_n$),
 to optimize $\mathcal{L}$ with respect to $\Theta$
we can use the derivative 
\begin{align}
\frac{\partial \mathcal{L}}{\partial \theta} = \frac{1}{2}
\mathrm{trace}(\tilde{\mat C}^{-1} \frac{\partial \tilde{\mat C}}{\partial
\theta}) - \frac{1}{2}\vec y^T \frac{\partial \tilde{\mat C}}{\partial \theta} \vec y,
\end{align}
where $\tilde{\mat C}=\mat L + \sigma^2 \mat I$.

On the other hand, the posterior distributions of the
hyperparameters given the data can be inferred by the Bayesian method via
Markov Chain Monte Carlo methods
similarly to \cite{Neal1997RC_GP,Williams95gaussian}. And the regression estimate can be averaged
over the hyperparameters rather than obtained by one fixed set of hyperparameters. This would produce better results but cost more computational power. This approach has not been taken in this paper

\section{Simulation Results}
\label{sec:sim_res} 

We did the simulations for the proposed methods to observe how the methods work and show their usefulness. The simulations for both methods were performed for both an artificial data set and a real world data set. 

\subsection{Simulation Results for Bayesian Kernel Regression}

First, we did the simulations for Bayesian kernel regression. To generate an artificial data set, we used the equation $\mathrm{sinc}(x) = \frac{\sin(\pi x)}{\pi x}$ for the sinc function. We took the points equally spaced with the interval 0.2  between -5 and 5. We made up the training set with those points as inputs and with the sinc function values at those points as target values.  And we took the points equally spaced with the interval 0.1 between -5.01 and 5. We made up the test set with those points as inputs and with the sinc function values at those points as target values. We call this training and test data set as the sinc data set I.  We made up another sinc data set (called the sinc data set II), which has the points equally spaced with the interval 0.5 (rather than 0.2)  between -5 and 5 as input points of the training set and the same test set as in the sinc data set I.

For the sinc data set I and II, we applied the proposed Bayesian kernel regression. For the comparison we also applied the traditional kernel regression with the bandwidth selected by leave-one-out cross-validation\footnote{Some functions in Econometrics package in Octave was used for kernel regression and bandwidth selection by leave-one-out cross-validation} \cite{hardle1985optimal}.  (We call this bandwidth as CV bandwidth.)  In addition we applied kernel regression with the bandwidth chosen in the proposed Bayesian kernel regression. (We call this bandwidth as B bandwidth.) For both data sets we tried the simulations repeatedly with different initial values for $\sigma_0,\sigma$, and found that one of the lowest marginal likelihoods is reached with the initial value 100, 1. 

Figure \ref{fig.1} shows the results of the three methods with the training data points for the sinc data set I. Apparently the proposed Bayesian kernel regression methods perform best.  In Figure \ref{fig.2} we can see how the hyperparameter can be selected in Bayesian kernel regression for the sinc data set I. They were plotted in two levels of scales to show that the selection of a specific value as well as averaging is reasonable for the hyperparameter because the evidence is highly peaked near the optimum as seen in Figure 2(b).

\begin{figure}
\centering
\includegraphics[width=8cm]{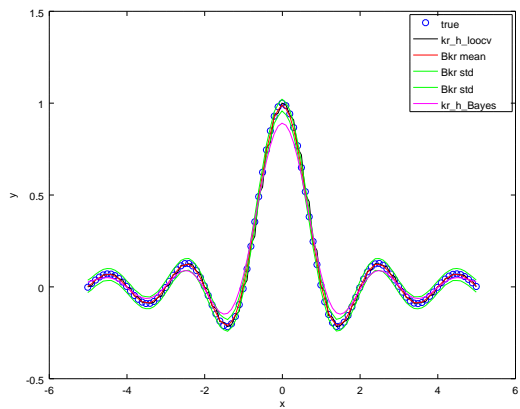}
\caption{The results of kernel regression with CV bandwidth, Bayesian kernel regression, and kernel regression with B bandwidth for the sinc data set I}
\label{fig.1}
\end{figure}

\begin{figure}
\centering
\begin{tabular}{c}
\includegraphics[width=8cm]{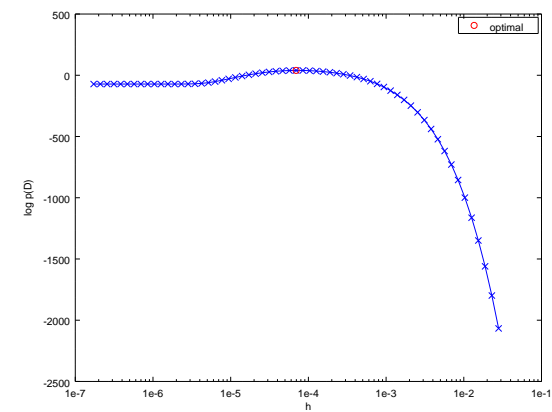} \\  
(a) \\
\includegraphics[width=8cm]{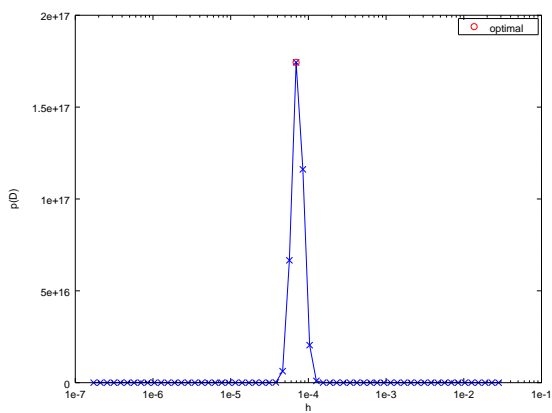} \\
(b)
\end{tabular}
\caption{ The use of evidence for hyperparameter selection in Bayesian kernel regression for the sinc data set I. Points x mean log evidence or evidence at some points. Circle points represent the optimum: (a) bandwidth vs. log evidence (b) bandwidth vs. evidence}
\label{fig.2}
\end{figure}

Table \ref{table:Bkr_sinc} shows the performance of the methods applied to the sinc data set I and II. The performance was measured in terms of mean squared error between true target values and regression estimates. Bayesian kernel regression was better than kernel regression with CV bandwidth or B bandwidth. 

\begin{table}[ht]
\caption{ Mean squared errors in kernel regression with CV bandwidth, Bayesian kernel regression, and kernel regression with B bandwidth for the sinc data set I and II}
\centering 
\begin{tabular}{c c c} 
\hline\hline 
 Methods & sinc data set I & sinc data set II   \\ [0.5ex] 
\hline 
Kernel regression with CV bandwidth& 9.1840e-04 &  0.0060448 \\ 
Bayesian kernel regression  & {\bf 3.5371e-05} & {\bf 0.0012617} \\
Kernel regression with B bandwidth & 0.0018597 & 0.018771 \\  [1ex] 
\hline 
\end{tabular}
\label{table:Bkr_sinc}
\end{table}

Next, we applied the three methods including the proposed method to a real world data set. We used  the yacht hydrodynamics data set \cite{Lichman:2013} for the evaluation. The data set is related to prediction of residuary resistance of sailing yachts at the initial design stage, which are very important for evaluating the performance of the ship and for estimating the required propulsive power.  It has 308 data points, each of which contains 6 inputs related to hull geometry coefficients and the Froude number, and a target value of residuary resistance per unit weight of displacement.  Input data were normalized before they were used for the evaluation of all three methods. 

  Since the data set is 6-dimensional, we have two choices for bandwidth configuration. A single identical bandwidth can be used for all the dimensions (i.e. single bandwidth), or an individual bandwidth can be used for each dimension (i.e. multiple bandwidths).  Bandwidth selections by both single and multiple bandwidth schemes were tried. The leave-one-out cross-validation for multiple bandwidths in the traditional kernel regression\footnote{ Some functions in Econometrics package in Octave was modified and used for the multiple bandwidth scheme in the traditional kernel regression.} is expensive,  but it was tried for comparison.

Table \ref{table:Bkr_Yacht} shows the performance of the three methods with the single bandwidth scheme. The 10 fold cross-validation was done for the performance evaluation. In the table means and standard deviations for the 10 fold cross-validation are shown. To speed up the optimisation and avoid unreasonable local minima we use the heuristic scheme to fix $\sigma_0$ to 1 and $\sigma$ to 0.0000001 and to update the bandwidth only.   Bayesian kernel regression and kernel regression with B bandwidth performed better than kernel regression with CV bandwidth. 

\begin{table}[ht]
\caption{ Means and standard deviations of mean squared errors in kernel regression with CV bandwidth, Bayesian kernel regression, and kernel regression with B bandwidth,
in the single bandwidth scheme for the yacht hydrodynamics data set}
\centering 
\begin{tabular}{c c } 
\hline\hline 
 Methods & MSE ($\mu \pm \sigma$) \\ [0.5ex] 
\hline 
Kernel regressison with CV bandwidth & 33.540 $\pm$ 23.507   \\ 
Bayesian kernel regression  & {\bf 33.529 $\pm$ 23.504} \\
Kernel regression with B bandwidth  & {\bf 33.529 $\pm$ 23.504}  \\  [1ex] 
\hline 
\end{tabular}
\label{table:Bkr_Yacht}
\end{table}

Table \ref{table:Bkr_m_Yacht} shows the performance of the three methods with the multiple bandwidth scheme. As in the single bandwidth scheme, the 10 fold cross-validation was done for the performance evaluation. In the table means and standard deviations for the cross-validation are shown.  We fix $\sigma_0$, $\sigma$ to the same values as in the single bandwidth scheme and update the bandwidths only. Comparing with the results in Table \ref{table:Bkr_Yacht} it is clear that the multiple bandwidth scheme is far better than the single bandwidth scheme.   Bayesian kernel regression and kernel regression with B bandwidth performed better than kernel regression with CV bandwidth. 

\begin{table}[ht]
\caption{ Means and standard deviations of mean squared errors in kernel regression with CV bandwidth, Bayesian kernel regression, and kernel regression with B bandwidth 
in the multiple bandwidth scheme for the yacht hydrodynamics data set}
\centering 
\begin{tabular}{c c } 
\hline\hline 
 Methods & MSE ($\mu \pm \sigma$)  \\ [0.5ex] 
\hline 
Kernel regressison with CV bandwidth & 1.3266   $\pm$ 1.1723  \\    
Bayesian kernel regression  & {\bf 1.1436 $\pm$  0.7894}  \\
Kernel regression with B bandwidth & {\bf 1.1436 $\pm$ 0.7894}  \\  [1ex] 
\hline 
\end{tabular}
\label{table:Bkr_m_Yacht}
\end{table}

\subsection{Simulation Results for Bayesian Mutual $k$-NN Regression}

Second, we did the simulations for Bayesian mutual $k$-NN regression.  We used the sinc data set I and II like in the simulation for Bayesian kernel regression.  For both the data sets, we applied the proposed Bayesian mutual $k$-NN regression. For the comparison we also applied the traditional $k$-NN regression and mutual $k$-NN regression where both $k$'s were selected by leave-one-out cross-validation.  (We call this $k$ as CV $k$.)  In addition we applied mutual $k$-NN regression with $k$ chosen in the proposed Bayesian mutual $k$-NN regression. (We call this $k$ as B $k$.) For both data sets we tried the simulation repeatedly with different initial values for $\sigma_0,\sigma$, and found that one of the lowest marginal likelihoods is reached with the initial value 300, 3. 

Figure \ref{fig.3} shows the results of Bayesian mutual $k$-NN regression and mutual $k$-NN regression with CV $k$ with the training data points for the sinc data set I. 
Figure \ref{fig.4} shows how the leave-one-out cross-validation method works for mutual $k$-NN regression for the sinc data set I. 
In Figure \ref{fig.5} we can see how the hyperparameter $k$ can be selected in Bayesian mutual $k$-NN regression for the sinc data set I. Marginal likelihoods were plotted in two kinds of scales to show that the selection of a specific value as well as averaging is reasonable for the hyperparameter $k$ because the evidence is highly peaked near the optimum as seen in Figure 5(b). 

\begin{figure}
\centering
\includegraphics[width=8cm]{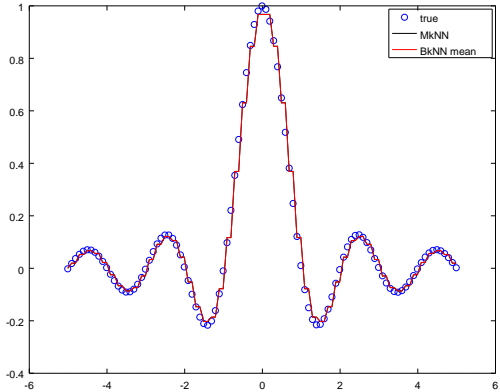}
\caption{The results of mutual $k$-NN with CV $k$, Bayesian mutual $k$-NN regression, mutual $k$-NN with B $k$ for sinc data set I}
\label{fig.3}
\end{figure}

\begin{figure}
\centering
\includegraphics[width=8cm]{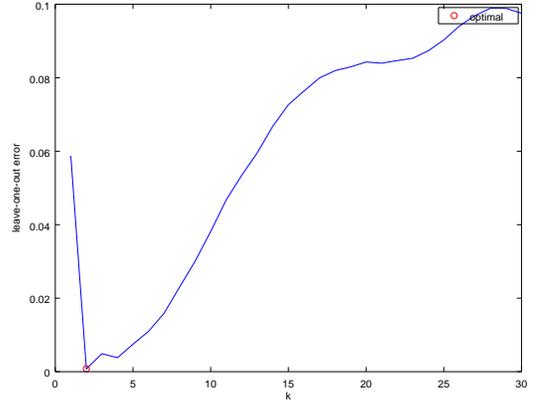}
\caption{The leave-one-out cros-validation to select $k$ for mutual $k$-NN regression for sinc data set I}
\label{fig.4}
\end{figure}

\begin{figure}
\centering
\begin{tabular}{c}
\includegraphics[width=8cm]{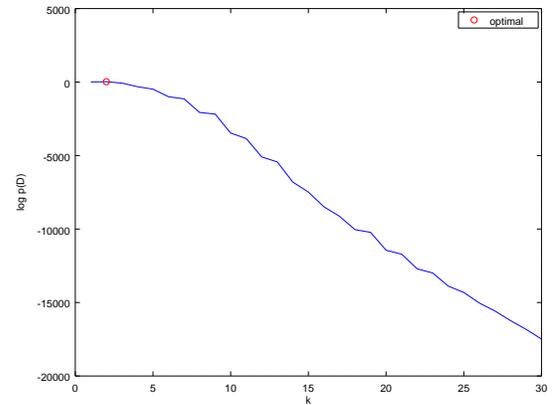} \\  
(a) \\
\includegraphics[width=8cm]{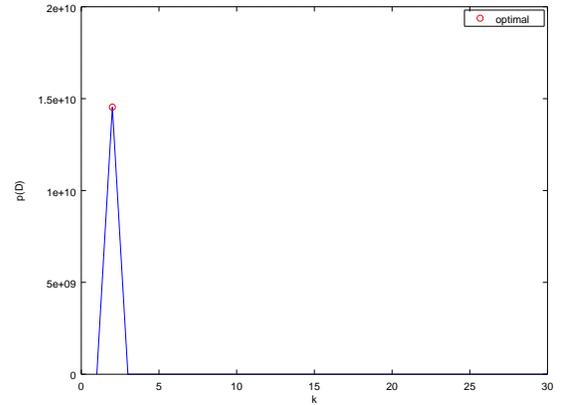} \\
(b)
\end{tabular}
\caption{The use of evidence for the selection of $k$ in Bayesian mutual $k$-NN regression for the sinc data set I. Points x mean log evidence or evidence at some $k$. Circle points represent the optimum: (a) $k$ vs. log evidence, (b) $k$ vs. evidence   }
\label{fig.5}
\end{figure}

Table \ref{table:BMkNN_sinc} shows the performance of the methods applied to the sinc data set I and II. The performance was measured in terms of mean squared error between true target values and regression estimates. For the sinc data set I, Bayesian $k$-NN regression was a little bit worse than mutual $k$-NN regression with CV $k$, and mutual $k$-NN regression with B $k$ was as good as mutual $k$-NN regression with CV $k$. For the sinc data set II, Bayesian $k$-NN regression was the best.

\begin{table}[ht]
\caption{ Mean squared errors in $k$-NN with CV $k$, mutual $k$-NN with CV $k$, Bayesian mutual $k$-NN regression, mutual $k$-NN with B $k$ for the sinc data set I and the sinc data set II}
\centering 
\begin{tabular}{c c c} 
\hline\hline 
 Methods & sinc data set I & sinc data set II  \\ [0.5ex] 
\hline 
$k$-NN regression  with CV $k$ & 0.0012619 & 0.0069711 \\
Mutual $k$-NN regression  with CV $k$ & {\bf 0.0012573} & 0.0069573 \\ 
Bayesian $k$-NN regression  & 0.0012587 & {\bf 0.0059893} \\ \ 
$k$-NN regression with B $k$  & {\bf 0.0012573} & 0.0060497  \\[1ex]  
\hline 
\end{tabular}
\label{table:BMkNN_sinc}
\end{table}

Next, we applied the four methods including the proposed method to a real world data set. Similarly to Bayesian kernel regression, we used the yacht hydrodynamics data set for the evaluation. Table \ref{table:BMkNN_Yacht} shows the performance of the three methods. The 10 fold cross-validation was done for the performance evaluation. In the table means and standard deviations for the cross-validation are shown.  The hyperparamters $\sigma_0,\sigma$ was fixed to 0.1, 0.00001, and only $k$ was selected.  All the three methods showed the same performance with the same $k(=2)$ selected.

\begin{table}[ht]
\caption{ Means and standard deviations of mean squared errors in $k$-NN regression with CV $k$,mutual $k$-NN regression with CV $k$, Bayesian mutual $k$-NN regression, and mutual $k$-NN regression with B $k$ for the yacht hydrodynamics data set}
\centering 
\begin{tabular}{c c } 
\hline\hline 
 Methods & MSE avg. \\ [0.5ex] 
\hline 
$k$-NN regression with CV $k$ & {\bf 39.127 $\pm$ 23.004 } \\ 
Mutual $k$-NN regression with CV $k$ & {\bf 39.127 $\pm$ 23.004 } \\  
Bayesian mutual $k$-NN regression  & {\bf 39.127 $\pm$ 23.004 } \\  
Mutual $k$-NN  regression with B $k$  & {\bf 39.127 $\pm$ 23.004 } \\[1ex]  
\hline 
\end{tabular}
\label{table:BMkNN_Yacht}
\end{table}

\section{Related Works}
\label{sec:related_works}

We discuss the former works related to the proposed methods in this paper. \cite{kim2013transductive} has proposed a transductive regression method with Gaussian processes  and applied it to object pose estimation. They has defined Laplacian kernel similarly to Laplaican-based covariance matrix in this paper and proposed a Bayesian method to select the hyperparameters.  

\cite{verbeek2006gaussian} used Gaussian field and Laplacians with a kernel related to $k$-NN and locally linear embedding. They applied their methods to facial pose estimation and object correspondence learning. They proposed active learning method based on entropy minimization, and a model selection scheme based on maximum likelihood. 

\cite{zhu2003semi} extended Gaussian field to Gaussian processes for semi-supervised classification.
They used graph Laplacians with various kinds of covariance functions including squared exponential and mutual $k$-NN type functions. The proposed a Bayesian method how to select the hyperparameters. They applied their methods to various binary classification problems.

None of the above works has shown the relationship between the Laplacian-based method and the traditional kernel and mutual $k$-NN regression.  In this paper we have shown that the means of predictive target values in our proposed methods asymptotically converge to the kernel and mutual $k$-NN regression estimate. We have also proposed Bayesian model selection methods for key hyperparameters in both estimates.  In \cite{kim2013transductive} and \cite{verbeek2006gaussian} symmetry as a covariance matrix related to $k$-NN was not checked, even though \cite{verbeek2006gaussian} used only symmetric matrixes for the simulation.  We have presented the theorems to show the validity of the covariance matrixes used in the proposed methods. 

\section{Conclusion}

We have proposed Bayesian kernel and mutual $k$-NN regression methods. Those two regression methods work in the framework of Gaussian process regression. Comparing to the traditional kernel regression and mutual $k$-NN regression, it has advantages to provide not only a distribution for the target value but also the principled way to select hyperparameters.  Even though the leave-one-out cross-validation can be done for the traditional methods, the performances of the proposed methods were better or comparble in terms of estimation accuracy and computational complexity. Especially, the multiple bandwidth scheme can be applied much more efficiently in Bayesian kernel regression than in traditional kernel regression. The simulation results for the artificial and real world data set show the superiority and efficacy of the proposed methods comparing with the traditional methods.

It is valuable to compare Gaussian process regression (GPR) with squared exponential covariance function (Eq (\ref{eqn:cov_ftn})) to Bayesian kernel regression (BKR), i.e. GPR with Laplacian-based covariance matrix. GPR has an advantage in terms of accuracy, and BKR has an advantage in terms  of computational complexity.  When we compare the formulations of means of target values in two models, in the one in BKR (Eq (\ref{eqn:result}))  only a limited correlation among data points is expressed, but in the one in GPR (Eq (\ref{eq:transd_lap_mean})) much more complicated correlation among data points is expressed. So regression performance in GPR should be better than the one in BKR.  However, unllike in GPR, BKR has advantages that it does not require the computation of matrix inversion, which is the computational bottleneck of GPR, for the mean and variance of target value (See Eq (\ref{eqn:result}) vs. Eq (\ref{eq:transd_lap_mean}), (\ref{eq:transd_lap_std})), and that it requires the computation of matrix determinant but not matrix inversion for the log evidence (See Eq (\ref{eqn:loglikeli}) vs. Eq (\ref{eqn:bkr_log_ev})). (It requires the computation of matrix inversion for the derivative of the log evidence.) 

There are some possible future works on the proposed methods. For Bayesian kernel regression, we have had a new bandwidth selection scheme. This scheme can be applied to many application areas including financial engineering and image processing. \cite{ait1998nonparametric,ait2000nonparametric}  used kernel regression with a heuristic bandwidth selection method in financial engineering. It should be interesting to use Bayesian kernel regression to their tasks, and to compare two results. Data-adapted extensions  of kernel regression such as bilateral filtering in image processing can be extended to Bayesian models like in the proposed Bayesian kernel regression so that it may have bandwidth or hyperparameter selection scheme. As future works for Bayesian mutual $k$-NN regression, it may be possible to build Bayesian mutual $k$-NN classifiers, or related models.

\newpage

\bibliographystyle{IEEEtran}
\bibliography{paper}

%








\end{document}